\documentclass[]{bytedance_seed}

\usepackage[toc,page,header]{appendix}

\usepackage{xspace}
\usepackage{mathastext}
\usepackage{amsfonts}

\usepackage{minitoc}

\usepackage{pifont}
\newcommand{\cmark}{\ding{51}}%
\newcommand{\xmark}{\ding{55}}%

\usepackage{amsthm}

\usepackage{wrapfig,lipsum,booktabs}

\theoremstyle{plain}
\newtheorem{theorem}{Theorem}[section]
\newtheorem{proposition}[theorem]{Proposition}
\newtheorem{lemma}[theorem]{Lemma}

\theoremstyle{definition}
\newtheorem{definition}[theorem]{Definition}

\theoremstyle{remark}

\DeclareRobustCommand{\modelname}{\textsc{CompTok}\xspace}
\DeclareRobustCommand{\armodelname}{$\epsilon$MaskGIT\xspace}

\usepackage{array}

\newlength\savewidth\newcommand\shline{\noalign{\global\savewidth\arrayrulewidth
  \global\arrayrulewidth 1pt}\hline\noalign{\global\arrayrulewidth\savewidth}}

\newcommand{\tablestyle}[2]{\setlength{\tabcolsep}{#1}\renewcommand{\arraystretch}{#2}\centering\footnotesize}
\definecolor{baselinecolor}{gray}{.9}

\title{Composable Visual Tokenizers with Generator-Free Diagnostics of Learnability}

\author[1,2]{Bingchen Zhao}
\author[1]{Qiushan Guo}
\author[1]{Ye Wang}
\author[1]{Yixuan Huang}
\authornewline
\author[1]{Zhonghua Zhai}
\author[1,\dagger]{Yu Tian}

\affiliation[1]{ByteDance Seed}
\affiliation[2]{University of Edinburgh}

\contribution[\dagger]{Corresponding authors}

\abstract{
We introduce \modelname, a training framework for learning visual tokenizers whose tokens are enhanced for compositionality.
\modelname uses a token-conditioned diffusion decoder.
By employing an InfoGAN-style objective, where we train a recognition model to predict the tokens used to condition the diffusion decoder using the decoded images, we enforce the decoder to not ignore any of the tokens.
To promote compositional control, besides the original images, \modelname also trains on tokens formed by swapping token subsets between images, enabling more compositional control of the token over the decoder.  
As the swapped tokens between images do not have ground truth image targets, we apply a manifold constraint via an adversarial flow regularizer to keep unpaired swap generations on the natural-image distribution.
The resulting tokenizer not only achieves state-of-the-art performance on image class-conditioned generation, but also demonstrates properties such as swapping tokens between images to achieve high level semantic editing of an image.
Additionally, we propose two metrics that measures the landscape of the token space that can be useful to describe not only the compositionality of the tokens, but also how easy to learn the landscape is for a generator to be trained on this space.
We show in experiments that \modelname can improve on both of the metrics as well as supporting state-of-the-art generators for class conditioned generation.
}

\date{\today}
\correspondence{Bingchen Zhao at \email{zhaobc.gm@gmail.com}, Yu Tian at \email{tianyu0124@gmail.com}}

\begin{document}
\maketitle

\begin{figure*}[t]
\centering
\includegraphics[width=\linewidth]{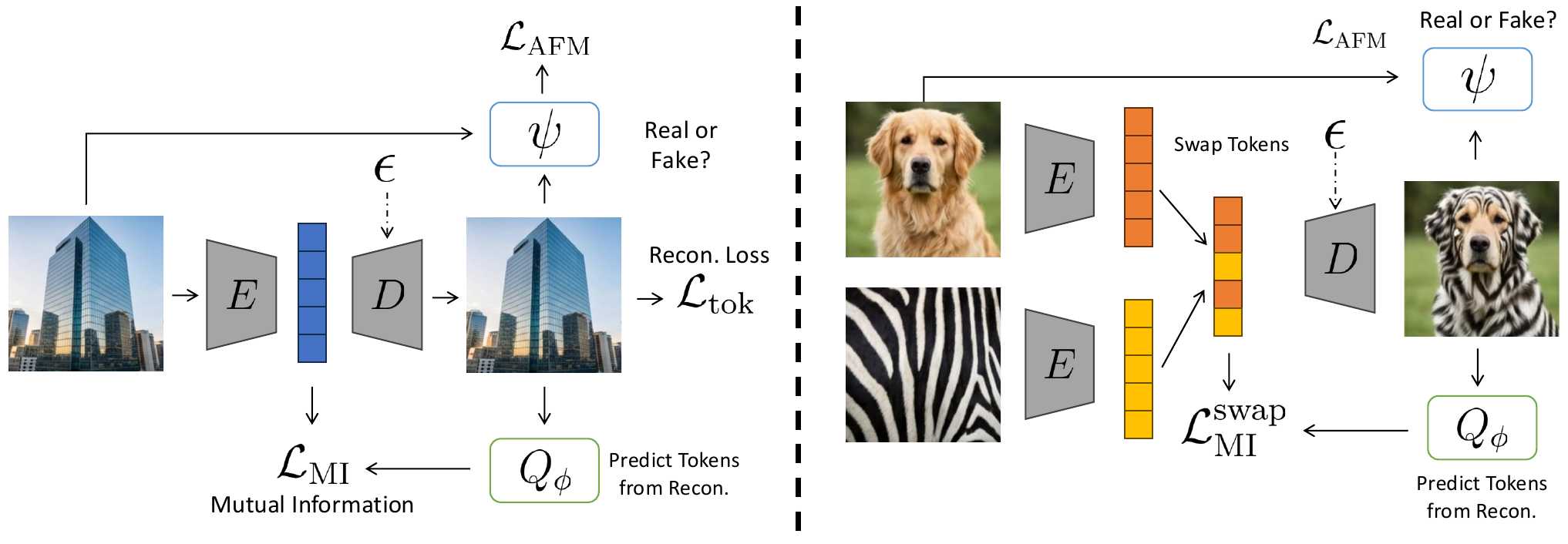}
\vspace{-1.5em}
\caption{\modelname overview. Left: a reconstruction pathway where the encoder $E$ produces a 1D token sequence that is decoded by $D$; a recognition head $Q_\phi$ must recover the tokens from the decoded image (mutual-information loss $\mathcal{L}_{\mathrm{MI}}$), while an adversarial flow/density model $\psi$ enforces realism via $\mathcal{L}_{\mathrm{AFM}}$ alongside the reconstruction loss $\mathcal{L}_{\mathrm{tok}}$. Right: a swap pathway where tokens from two images are partially exchanged, the mixed token is decoded, and the same recoverability and realism constraints $\mathcal{L}_{\mathrm{MI}}^{\mathrm{swap}}$ and $\mathcal{L}_{\mathrm{AFM}}$ train tokens to be non-degenerate and compositional under edits.
}
\label{fig:method}
\end{figure*}

\section{Introduction}
Visual tokenizers define the representation space in which modern image generators operate. 
Most established tokenizers produce 2D latent grids, spatial feature maps that preserve locality and are naturally consumed by convolutional or U-Net style decoders.
Increasingly, however, recent work has explored 1D token sequences that linearize visual content into a list of tokens, motivated by tighter integration with transformer-based priors~\cite{titok}, flexible rate allocation~\cite{flextok}, and a more “language-like” interface for autoregressive modeling and editing~\cite{semanticist}. 
This shift raises a foundational question: what properties should a tokenizer’s tokens satisfy in order to serve as a reliable representation space for generation and manipulation?

For tokenizers to be broadly useful beyond reconstruction, their tokens should satisfy two properties. 
First, token usefulness: each token (or token group) should carry non-trivial, usable information about the image, rather than being redundant given other tokens~\cite{vqganlc}. 
Prior works~\cite{vqganlc} studied that the utilization of codebooks matters for the performance of tokenizers, indicating that we need the tokens to carry as much of useful information as possible.
Second, token compositionality: tokens should behave as stable control variables, swapping or editing a subset of tokens should yield predictable changes while keeping the decoded image on the natural-image manifold.
\citet{beyer2025highly} shows that 1D tokenizers can be used directly for generation tasks because it is highly compressed so the tokens can demonstrate stable control over the image.
These properties directly determine whether the token from the tokenizer can support training of generators for fine-grained control and modular editing.

We study 1D tokenizers in this work, as they are the instances where the tokens carry semantic level information over the images rather than patch level information~\cite{titok,beyer2025highly,semanticist}.
Existing work has largely optimized 1D tokenizers for rate–distortion~\cite{titok} and modeling efficiency~\cite{vqganlc,llamagen}, and has demonstrated that imposing structure can yield strong generation with relatively few tokens~\cite{flextok,semanticist}. 
This is valuable, but it also highlights what is missing: current evaluations and training objectives do not explicitly enforce that \emph{all} tokens remain causally effective and compositional.
In other words, the tokenizers can learn highly entangled tokens under these constriants, and these complex interactions between tokens can make it hard for generators to learn on the token space.
As a result, reconstruction fidelity can be a poor predictor of generative performance and of controllability as shown in prior works~\cite{vavae}: a tokenizer may reconstruct well yet induce a latent space whose ``usable region" is fragmented or whose subset of tokens are effectively ignored by the generator~\cite{flextok,semanticist}.

In this work, we propose \modelname, a 1D tokenizer training framework designed to produce compositional tokens that provides a easier to learn landscape for the generators.
\modelname leverages a diffusion-based decoder following prior works~\cite{flextok,semanticist} as this is able to decouple the frequency information from the tokens~\cite{semanticist}.
Additionally, \modelname uses an InfoGAN-style loss that requires the decoded images from the decoder can be used by a recognition model to predict the exact tokens the images is decoded from.
This requires the tokens to be learned in a causally effective manner where the decoder cannot ignore them, otherwise it will get penalized.
To further strength the compositional control of tokens over the decoder, the decoder of \modelname trains on tokens formed by swapping token subsets between images, and enforces realism for these compositions of tokens via a manifold constraint implemented with an adversarial flow regularizer~\cite{lin2025adversarial}.
This swap-regularized training expands the set of valid token compositions and strengthens the interventional effect of tokens, enabling strong performance in class-conditioned image generation.

Our second contribution is a pair of metrics for tokenizer usability that go beyond reconstruction fidelity. 
Average information gain (AvgIG) measures how reconstruction progress is distributed during latent optimization starting with a random token to reconstruct a given input image.
AvgIG measures on average, how much information is gained by each step of optimization. If the landscape is flat around a real token, AvgIG will be small as optimization gains less information at each step. This implies the decoder is ignoring token changes, a phenomenon we term Token Neglect. Consequently, a higher AvgIG is desirable to ensure the decoder remains sensitive to token changes, thereby guaranteeing that every token contributes effectively.
Mode connectivity (MC) characterizes latent geometry by measuring whether high-quality real tokens, and critically, mixed tokens produced by token swapping, lie in a connected, low-barrier region, indicating whether the representation is learnable for a generator and stable under compositional edits. 
If the landscape has high walls or spikes between two valid real tokens, the space is fragmented. 
This makes it impossible for a generator to interpolate or transition between concepts encoded by different tokens.
A higher MC would indicate that the representation avoids fragmentation and is stable under compositional edits.
Crucially, these metrics are complementary: AvgIG ensures the ``basins" of the loss landscape are steep and well-defined (preventing token neglect), while MC ensures these basins are connected by smooth, low-loss paths (preventing manifold artifacts).

Empirically, we evaluate tokenizers built from diverse methods~\cite{rae,VQVAE}. 
We show that AvgIG and MC have substantial predictive power for downstream generation and control, and that \modelname consistently improves both metrics and generative performance—demonstrating that practical tokenizer quality is governed not only by reconstruction, but by token utilization and the geometry of compositional latent space.

\section{Related Work}
\label{sec:related_work}

\noindent \textbf{Image tokenization for efficient generative modeling.}
Visual tokenizers compress images into latent tokens that make training and sampling feasible for large generators. 
Classical approaches include VAEs~\cite{vae} and VQ-VAE~\cite{VQVAE}, which introduced a learnable codebook for mapping images to discrete tokens. 
Subsequent work improved perceptual fidelity and sample sharpness via adversarial and perceptual losses~\cite{VQGAN} and increased representational capacity via multi-stage or residual quantization~\cite{rqvae}. 
Codebook under-utilization and commitment pathologies have been actively studied; MAGVIT-v2~\cite{magvitv2} proposes Look-up Free Quantization (LFQ) to mitigate these issues. 
More recent tokenizers modernize architectures and training losses, including binary quantization and large-scale conditional generation~\cite{maskbit} and scaling studies for ViT-based tokenizers~\cite{vitok,vit}. 
A complementary line leverages pretrained visual foundation models (e.g., CLIP/MAE/DINOv2) to enrich token semantics or stabilize training~\cite{radford2021clip,MAE,dinov2,vqganlc,lightningdit}.

\noindent \textbf{From 2D grids to 1D causal token sequences.}
Many tokenizers produce 2D grids of latent tokens aligned to spatial patches, which is convenient for convolutional decoders but does not directly provide a causal ordering for autoregressive modeling. 
Several works propose 1D tokenizations that impose a causal dependency, enabling LLM-style decoding for images and more natural integration with autoregressive backbones. 
TiTok~\cite{titok} and SEED~\cite{SEED} are representative early 1D tokenizers; subsequent work studies how causal ordering benefits autoregressive modeling~\cite{CRT}.
VAR~\cite{var} reframes visual AR modeling as next-scale prediction and introduces a multi-scale residual quantization tokenizer that produces low-to-high resolution codes. 
Adaptive-length tokenization has also been explored, where token budget depends on image content or downstream needs~\cite{alit}. 
These directions largely optimize for compression, reconstruction, and likelihood under a chosen generator, while the \emph{meaning and controllability of individual tokens} remain less explicitly constrained.

\noindent \textbf{Token usefulness and controllable manipulation.}
A recurring practical gap is that strong reconstructions do not guarantee that tokens are \emph{beneficial} for a downstream generator~\cite{vavae}, nor that token edits stay on the natural-image manifold~\cite{beyer2025highly}. 
Prior work on ordered tokens~\cite{semanticist,flextok} demonstrates that token prefixes can suffice for generation and can improve efficiency, but such results do not directly address our objective: we treat tokens as \emph{compositional control variables} and prioritize (i) non-degeneracy (each token has interventional effect) and (ii) validity under mixing (swaps remain realistic and stable). 
\modelname differs in that it explicitly targets token \emph{usefulness} and \emph{compositionality} through training objectives and diagnostics, rather than focusing primarily on adaptive budget or prefix sufficiency.

\noindent \textbf{Diffusion and autoregressive generators with token conditioning.}
Modern high-fidelity image generators are dominated by diffusion-based models~\cite{DDPM,ddim,dhariwal2021diffusion,dit} and large autoregressive models~\cite{pixelcnn,pixelrnn,llamagen}. 
Latent diffusion~\cite{LDM} popularized performing diffusion in the latent space of a pretrained tokenizer~\cite{VQGAN}, substantially reducing compute and enabling scalable text/image-conditional generation. 
Autoregressive approaches require an ordering over tokens; both fixed orderings and randomized orders have been explored~\cite{maskgit,RAR,mar}. 
These models typically assume a tokenizer as a front-end, which provides the data distribution to train the generation model on.
\modelname is a tokenizer and aim to make the tokenizer directly token neglect and encourage compositional control, aiming to improve downstream generation and editing rather than reconstruction alone.

\section{Method}
\label{sec:method}

\begin{figure*}[t]
\centering
\includegraphics[width=\linewidth]{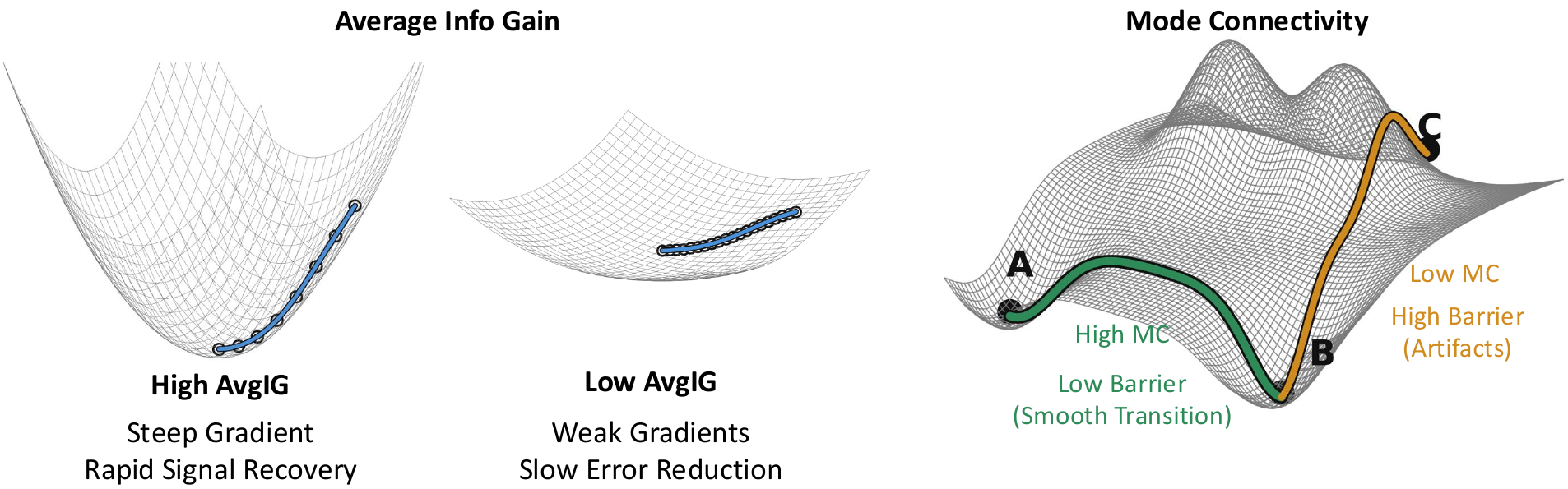}
\vspace{-1.0em}
\caption{AvgIG and MC are computed on the tokenizer (E,D) to evaluate the quality of the tokenizer and the downstream generator. AvgIG measures the average information gain per optimization step when fitting a latent token to a target image, and MC measures the pairwise mode connectivity of the tokenizer token space. MC is a measure of the local geometric smoothness of the tokenizer token space.}
\label{fig:metrics}
\end{figure*}

\subsection{Problem Setting}
\label{sec:setting}

We study a two-stage pipeline: a learned visual tokenizer and a downstream conditional generator trained on the tokenizer's token space.
The tokenizer consists of an encoder--decoder pair $(E,D)$ mapping images $x\in\mathcal{X}$ to a structured latent token $z\in\mathcal{Z}$ and back:
$z = E(x)$, $\hat{x} = D(z)$.
We represent the token as a 1D token sequence (or token groups) $z=(z_1,\dots,z_K)$, and we do not perform vector quantatization following~\citet{semanticist}.
We use $\mathcal{X}$ to denotes the full space of all real images, and in practice for training tokenizers we would have a dataset $\mathcal{D}$ to sample data from.
In this paper, $D$ is implemented as a \emph{diffusion-based decoder} that reconstructs $x$ conditioned on $z$, following prior diffusion-decoder tokenizers~\cite{semanticist}.
Concretely, $D$ is trained to reconstruct $x$ from $z$ via a conditional denoising objective~\cite{DDPM}.

Separately, we train a downstream generator $G$ \emph{on the token space} produced by the tokenizer.
$G$ can be an image generator conditioned on class labels or text, generating the image tokens to be decoded by $D$ for pixels.
Critically, $G$ is \emph{not} part of the tokenizer.
It is trained after $(E,D)$ and its success depends on the \emph{quality and geometry} of the token space $\mathcal{Z}$ induced by the tokenizer.

The most common evaluations on $(E,D)$ are to measure the Fréchet Inception Distance (FID) of the reconstructed images~\cite{fid} to a reference set.
The evaluation on $G$ is done similarly by measuring the FID of the generated images by sampling from $G$ conditioned on class labels.
We term the FID of the reconstructed images as rFID and the FID of the generated images as gFID.

\subsection{Why rFID Can Mislead on gFID}
\label{sec:rfid_gap}

A central empirical issue is that tokenizer reconstruction quality (e.g., rFID of $D(E(x))$) is not always predictive of downstream generation quality of $G$~\cite{vavae,vitok}.
Our core view is that rFID evaluates reconstructions only on the \emph{encoder manifold} on a test dataset $\mathcal{M}_E=\{E(x):x\sim\mathcal{D}_{\text{test}}\}$.
However, $G$ must learn a distribution over \emph{valid tokens} and produce high-quality images from them.
This introduces two failure modes that rFID does not directly probe:
{(1) Token-space completeness of the decoder $D$.}
Even if $E$ is fixed, rFID does not separate whether $D$ has the capacity to faithfully reconstruct the target distribution for \emph{arbitrary} valid tokens.
If good reconstructions require finely tuned tokens that occupy a thin set that only $E$ can generate, then $G$ faces a difficult support-learning problem.
{(2) Learnability / geometry of the good-token region.}
If high-quality tokens form multiple disconnected pockets or have high barriers between modes, then a generator $G$ trained on tokens will struggle to model this geometry.
This yields a mismatch between the token distribution that $G$ can learn and the tokens that decode well, leading to degraded gFID even when rFID is strong.
These observations motivate an evaluation framework for the tokenizer that predicts the downstream potential of $G$ \emph{without training $G$}.

\subsubsection{Metrics: AvgIG and MC}
\label{sec:diagnostics}

We introduce two metrics for tokenizer quality that go beyond reconstruction-only metrics such as rFID.
Standard reconstruction metrics evaluate samples only on the encoder manifold; they do not measure whether the token space forms a geometry that a downstream generator can actually learn.
Our metrics address this by probing two complementary properties: \textbf{AvgIG} ensures the decoder is \emph{locally sensitive} to tokens, while \textbf{MC} ensures the latent manifold is \emph{globally connected}.
Both are computed without training the downstream generator $G$.

\noindent \textbf{AvgIG}: This metric quantifies the ``vertical'' steepness of the loss landscape, measuring how much \emph{usable} information the diffusion decoder $D$ can absorb per optimization step towards reconstructing an image from a random latent token.
Intuitively, for tokens to be causally effective, the decoder must be highly sensitive to their values.
If decoding is well-conditioned, optimization should yield a rapid reduction in reconstruction error, indicating that the information in the tokens $z$ is actionable and accessible, a prerequisite for downstream learning.

\textbf{Definition.}
Fix the diffusion decoder $D$ and a target image $x$.
Starting from an initialization $z^0$, we optimize the latent token by gradient descent on MSE:
\begin{align}
z^{t+1} \;&=\; z^t - \eta \nabla_z \tfrac{1}{2}\|D(z^t)-x\|_2^2 \nonumber\\
\mathrm{MSE}_t \;&=\; \tfrac{1}{N}\|D(z^t)-x\|_2^2,
\end{align}
where $N$ is the number of pixels.
We define the per-step information gain (in bits):
\begin{align}
\Delta I_t(x) \;&=\; \frac{N}{2}\log_2\!\left(\frac{\mathrm{MSE}_t}{\mathrm{MSE}_{t+1}}\right) \nonumber\\
\mathrm{AvgIG} \;&=\; \mathbb{E}_{x\sim\mathcal{D}}\big[\Delta I_t(x)\big].
\label{eq:avgig}
\end{align}
AvgIG measures the efficiency of signal recovery.
A low AvgIG implies weak gradients or flat local plateaus where changing $z$ has little effect on $D(z)$, a symptom of token neglect.
High AvgIG indicates that tokens act as stiff control variables, facilitating the learning of a precise generator distribution.
Under a Gaussian residual model, $\Delta I_t$ is proportional to the reduction in residual entropy (bits), serving as an operational measure of ``usable information''.

\noindent \textbf{MC}:
While AvgIG demands local steepness, MC demands ``horizontal'' flatness along the data manifold.
We measure whether \emph{nearby images} map to tokens connected by a \emph{low-loss path} in the tokenizer's token space.
If the token space is fragmented, interpolating between valid tokens will pass through high-loss barriers (artifacts), making the space difficult for $G$ to model.

\textbf{Definition.}
We sample pairs $(x^A,x^B)$ such that $d_{\mathcal{X}}(x^A,x^B)\le \epsilon$, where $d_{\mathcal{X}}$ is a distance measure defined by the image space and $\epsilon$ is a hyperparameter and encode $z^A=E(x^A)$ and $z^B=E(x^B)$.
Let $\gamma:[0,1]\to\mathcal{Z}$ be a path connecting the tokens (linear interpolation), $\gamma(t)=(1-t)z^A+t z^B$.
Instead of pixel-wise reconstruction error, we evaluate the path using a realism loss $\mathcal{L}_{\psi}$ derived from an auxiliary density model $p_{\psi}$ (where higher loss $\mathcal{L}_{\psi}(x)$ indicates lower realism).
We define the worst realism loss along the latent path, $L_{\max}=\max_{t\in[0,1]} \mathcal{L}_{\psi}(D(\gamma(t)))$, and the best reference loss of the endpoints ($L_{\text{ref}} \;=\; \min\!\Big(\mathcal{L}_{\psi}(D(z^A)),\,\mathcal{L}_{\psi}(D(z^B))\Big)$
We define the \textbf{pairwise mode connectivity} score as the ratio:
\begin{equation}
\mathrm{MC}(x^A,x^B)
\;=\;
\frac{L_{\text{ref}}}{L_{\max}+\delta}
\;\in\;[0,1],
\label{eq:mc_ratio_pair}
\end{equation}
where $\delta>0$ ensures numerical stability.
This score approaches $1$ if the path never incurs a realism loss higher than the better endpoint (no barrier), and drops toward $0$ if the path traverses high-loss regions (off-manifold artifacts).
Finally, we report the dataset-level score by sampling different $x^A$ and $x^B$ from the dataset that $d_{\mathcal{X}}(x^A,x^B)\le \epsilon$:
\begin{equation}
\mathrm{MC}=\mathbb{E}\big[\mathrm{MC}(x^A,x^B)\big].
\label{eq:mc_dataset_pair}
\end{equation}
High MC indicates that the valid tokens region is geometrically smooth and connected, ensuring that compositional edits and generative interpolation remain on the natural image manifold.

\subsection{Diffusion Decoder D}
\label{sec:tok_train}

We train the tokenizer $(E,D)$ for reconstruction.
Let $z=E(x)$ and let $x_t$ be the noisy image at diffusion step $t$ under a fixed schedule, with $\epsilon\sim\mathcal{N}(0,I)$.
The diffusion decoder $D$ is a denoiser $\epsilon_D=\epsilon_D(x_t,t;z)$.
trained by the standard conditional denoising loss:
\begin{equation}
\mathcal{L}_{\mathrm{tok}}(E,D)=\mathbb{E}_{x\sim\mathcal{D}}\Big[\|\epsilon-\epsilon_D(x_t,t;z)\|_2^2\Big].
\label{eq:ltok}
\end{equation}
This objective optimizes reconstruction under the encoder manifold $\mathcal{M}_E$; it does not directly enforce that good tokens form a learnable geometry or that tokens are evenly useful, motivating our metrics and the additional regularizers in \modelname.

\subsection{Aligning Tokenizer Training with Completeness and Learnability}
\label{sec:modeling}

Our evaluation suggests two properties of the tokenizer that matter for $G$:
(i) tokens should carry \emph{usable} information that effectively captures the information in the image,
(ii) the set of tokens that decode to realistic images should be low-barrier which is crucial for the learnability of the token space.
\modelname adds objectives that directly target these properties while preserving reconstruction quality.

\subsubsection{Mutual-Information Token Supervision}
\label{sec:mi}
To discourage token neglect, we introduce a recognition model $Q_\phi(z\mid x)$ trained to recover tokens from images decoded by the tokenizer decoder.
We optimize:
\begin{equation}
\mathcal{L}_{\mathrm{MI}}(D,\phi)=-\sum_{k=1}^K \mathbb{E}_{\hat{x}\sim D(\cdot\mid z)}\big[\log Q_\phi(z_k\mid \hat{x})\big],
\label{eq:lmi}
\end{equation}
where $D(\cdot\mid z)$ denotes sampling from the diffusion decoder conditioned on $z$.
AvgIG measures how effectively reconstruction error can be reduced by moving in token space during latent optimization.
If some tokens are ignored by $D$, then changing those coordinates has little effect on $D(z)$, yielding weak gradients and slow error reduction, i.e., low AvgIG.
Mutual-information supervision prevents this failure mode by forcing $D(\cdot\mid z)$ to produce images from which each $z_k$ is predictable.
Operationally, this increases the sensitivity of the decoded image to token perturbations, improving conditioning of latent optimization and thereby increasing AvgIG.

\subsubsection{Swap-Regularized Compositionality}
\label{sec:swap}
To enforce learnable geometry under token composition, we train on swapped tokens.
Given two images $x^A, x^B$, we extract their latent tokens as sequences of tokens, $z^A=(z^A_1, \dots, z^A_K)$ and $z^B=(z^B_1, \dots, z^B_K)$.
We form a swapped tokens $z^{A\leftarrow B}$ by replacing a subset of tokens in $z^A$ with corresponding tokens from $z^B$ according to a swap mask or index set $\mathcal{K}$.
For example, an interleaved swap might yield a mixed sequence:
\begin{equation}
z^{A\leftarrow B} \;=\; (z^A_1,\, z^B_2,\, z^A_3,\, \dots,\, z^B_K).
\end{equation}
We decode a swapped sample $\tilde{x}\sim D(\cdot\mid z^{A\leftarrow B})$ and apply the mutual-information token supervision objective to ensure the decoder respects the mixed composition:
\begin{align}
\mathcal{L}_{\mathrm{MI}}^{\mathrm{swap}}(D,\phi)
\;=\;
-\sum_{k=1}^K 
\mathbb{E}
\big[\log Q_\phi(z_k^{A\leftarrow B}\mid \tilde{x})\big].
\label{eq:lmi_swap}
\end{align}
This objective explicitly expands the set of valid compositions and empirically increases MC.

\begin{table*}[t]
\centering 
\resizebox{\linewidth}{!}{
\tablestyle{5.5pt}{1.05}
\begin{tabular}{l ccc |ccc | cccccc}
    Method & \#Token & Dim. & VQ & rFID$\downarrow$ & AvgIG$\uparrow$ & MC$\uparrow$ & Gen. Model & Type & \#Token & \#Step & gFID$\downarrow$ & IS$\uparrow$ \\
    \shline
    \textbf{2D-grid tokenizers} &&&&&&&&&&&&\\
    MaskBit~\cite{maskbit} & 256 & 12 & \cmark & 1.61 & 0.33 & 0.68 & MaskBit & Mask. & 256 & 256 & 1.52 & 328.6 \\
    MAR~\cite{mar} & 256 & 16 & \xmark & 1.22 & 0.36 & 0.70 & MAR-L & Mask. & 256 & 64 & 1.78 & 296.0 \\
    VQGAN~\cite{VQGAN} & 256 & 16 & \cmark & 7.94 & 0.19 & 0.59 & Tam. Trans. & AR & 256 & 256 & 5.20 & 280.3 \\
    ViT-VQGAN~\cite{vitvqgan} & 1024 & 32 & \cmark & 1.28 & 0.20 & 0.58 & VIM-L & AR & 1024 & 1024 & 4.17 & 175.1  \\
    RQ-VAE~\cite{rqvae} & 256 & 256 & \cmark & 3.20 & 0.18 & 0.57 & RQ-Trans. & AR & 256 & 64 & 3.80 & 323.7 \\
    VAR~\cite{var}  & 680 & 32 & \cmark & 0.90 & 0.29 & 0.66 & VAR-$d$16 & VAR & 680 & 10 & 3.30 & 274.4 \\
    ImageFolder~\cite{imagefolder} & 286 & 32 & \cmark & 0.80 & 0.31 & 0.67 & VAR-$d$16 & VAR & 286 & 10 & 2.60 & 295.0\\
    LlamaGen~\cite{llamagen} & 256 & 8 & \cmark & 2.19 & 0.17 & 0.44 & LlamaGen-L & AR & 256 & 256 & 3.80 & 248.3 \\
    CRT~\cite{CRT} & 256 & 8 & \cmark & 2.36 & 0.32 & 0.66 & LlamaGen-L & AR & 256 & 256 & 2.75 & 265.2 \\
    Causal MAR~\cite{mar} & 256 & 16 & \xmark & 1.22 & 0.29 & 0.64 & MAR-L & AR & 256 & 256 & 4.07 & 232.4 \\
    \hline
    \textbf{1D token-sequence tokenizers} &&&&&&&&&&&&\\
    TiTok-S-128~\cite{titok} & 128 & 16 & \cmark & 1.71 & 0.40 & 0.75 & MaskGIT-L & Mask. & 128 & 64 & 1.97 & 281.8 \\
    TiTok-L-32~\cite{titok} & 32 & 8 & \cmark & 2.21 & 0.27 & 0.62 & MaskGIT-L & Mask. & 32 & 8 & 2.77 & 194.0 \\
    Semanticist~\cite{semanticist} & 256 & 16 & \xmark &  0.78  & 0.38 & 0.73 & $\epsilon$LlamaGen-L & AR & 32 & 32 & 2.57 & 260.9  \\
    FlexTok~\cite{flextok} & 256 & 18 & \cmark & 1.45 & 0.35 & 0.71 & AR & AR & 32 & 32 & 1.86 & - \\
    \hline
    \textbf{\modelname} &  256 & 16 & \xmark & 1.27 & 0.43 & 0.74 & \armodelname & Mask. & 256 & 64 & 1.60 & 297.9 
\end{tabular}
}
\vspace{.5em}
\caption{Reconstruction and generation performance on ImageNet. ``Dim." denotes the dimension of the tokens, and ``\#Step" denotes the number of steps needed for generating the complete image. ``\#Token'' stands for the number of tokens used for image reconstruction (left) and generation (right), respectively.}%
\label{tab:tok_comp}
\end{table*}

\subsubsection{Manifold Constraint via Adversarial Flow Model (AFM)}
\label{sec:afm}

Token swaps produce \emph{unpaired} tokens $z^{A\leftarrow B}$ that often decode to off-manifold artifacts, creating connectivity barriers between valid regions.
To constrain these mixed tokens to the natural image manifold, we adopt the Adversarial Flow Model (AFM)~\cite{lin2025adversarial} as a realism prior.
We employ a time-conditioned discriminator $D_\psi(x_t, t)$ trained via a relativistic objective to distinguish between real data flows and generated flows.
The decoder $D$ is optimized to fool the discriminator while maintaining deterministic transport stability via the following generator objective:
\begin{align}
\mathcal{L}_{\mathrm{AFM}}(D) &=
\underbrace{\mathbb{E}\big[\text{softplus}(D_\psi(x_t, t) - D_\psi(\tilde{x}_t, t))\big]}_{\text{Manifold Constraint (Realism)}}\\
&+\lambda_{\mathrm{OT}}\underbrace{\mathbb{E}\big[\|\tilde{x} - x\|_2^2\big]}_{\text{Stability (OT)}},
\label{eq:lafm}
\end{align}
where $\tilde{x}$ are decoded images (reconstructions of swapped tokens or non-swapped tokens) diffused to time $t$, and $x$ are real targets (applied only when available).
The adversarial term acts as a powerful manifold prior for swapped tokens $z^{A\leftarrow B}$: it penalizes any decode $\tilde{x}$ that lies in low-density artifact regions, effectively pushing the decoder to map mixed tokens onto the natural image manifold defined by $D_\psi$.
Simultaneously, the Optimal Transport (OT) term ensures the mapping remains straight and deterministic where targets exist.
By forcing swapped decodes to be realistic, AFM eliminates the high-error barriers that typically separate valid tokens, thereby smoothing the latent geometry and directly increasing Mode Connectivity (MC).

\modelname is trained with a combined loss of the mutual information loss, the swapped mutual information loss, and the AFM loss.

\begin{table*}[t]
  \vspace{-2mm}
  \centering
  \setlength{\tabcolsep}{4.5pt}
  \renewcommand{\arraystretch}{1.12}
  \resizebox{\textwidth}{!}{
  \begin{tabular}{l|ccc|c|ccc|c|ccc|c}
  \hline
  \multirow{2}{*}{\textbf{Tokenizer}} &
  \multicolumn{4}{c|}{\textbf{Remote Sensing}} &
  \multicolumn{4}{c|}{\textbf{Medical}} &
  \multicolumn{4}{c}{\textbf{Text OCR}} \\
  \cline{2-13}
  & \textbf{rFID}$\downarrow$ & \textbf{AvgIG}$\uparrow$ & \textbf{MC}$\uparrow$ & \textbf{Task}$\uparrow$
  & \textbf{rFID}$\downarrow$ & \textbf{AvgIG}$\uparrow$ & \textbf{MC}$\uparrow$ & \textbf{Task}$\uparrow$
  & \textbf{rFID}$\downarrow$ & \textbf{AvgIG}$\uparrow$ & \textbf{MC}$\uparrow$ & \textbf{Task}$\uparrow$ \\
  \hline
  \textbf{2D-grid tokenizers} &&&&&&&&&&&&\\
  \hline
  RAE~\cite{rae}                 & 146.1 & 0.130 & 0.47 & 75.9 & 160.7 & 0.124 & 0.46 & 59.3 & 131.5 & 0.121 & 0.49 & 76.0 \\
  VQGAN~\cite{VQGAN}             & 162.9 & 0.149 & 0.54 & 83.4 & 161.3 & 0.135 & 0.48 & 63.5 & 141.4 & 0.138 & 0.51 & 79.4 \\
  MAE~\cite{MAE}                 & 135.3 & 0.147 & 0.54 & 85.1 & 155.4 & 0.141 & 0.54 & 62.0 & 134.3 & 0.137 & 0.54 & 80.4 \\
  SDXL VAE~\cite{podell2023sdxl} & 145.9 & 0.131 & 0.48 & 79.7 & 158.2 & 0.128 & 0.46 & 62.3 & 137.2 & 0.143 & 0.53 & 79.2 \\
  \hline
  \textbf{1D token-sequence tokenizers} &&&&&&&&&&&&\\
  \hline
  TiTok-L-32~\cite{titok} & 157.0 & 0.159 & 0.64 & 86.5 & 174.1 & 0.131 & 0.60 & 63.5 & 143.6 & 0.160 & 0.66 & 87.6 \\
  Semanticist~\cite{semanticist} & 159.9 & 0.144 & 0.64 & 84.2 & 171.6 & 0.151 & 0.66 & 67.0 & 159.0 & 0.144 & 0.66 & 86.3 \\
  FlexTok~\cite{flextok}         & 154.8 & 0.154 & 0.63 & 86.7 & 160.8 & 0.162 & 0.66 & 71.6 & 155.3 & 0.163 & 0.65 & 87.8 \\
  \hline
  \textbf{\modelname (ours)}   & 168.1 & 0.157 & 0.66 & 88.7 & 153.7 & 0.159 & 0.69 & 74.9 & 141.5 & 0.167 & 0.65 & 91.9 \\
  \hline
  Original Image &  -    & -     &   -  &  89.5 & -     & -    & -    & 76.5 &  -    &  -    &  -    & 92.5 \\
  \hline
  \multicolumn{13}{l}{\textbf{Correlation with Task Utility (Pearson $r$ across tokenizers; higher is better)}} \\
  \hline
  $r(\text{rFID},\text{Task})$   & \multicolumn{4}{c|}{0.48} & \multicolumn{4}{c|}{-0.18} & \multicolumn{4}{c}{0.62} \\
  $r(\text{AvgIG},\text{Task})$  & \multicolumn{4}{c|}{0.93} & \multicolumn{4}{c|}{0.91} & \multicolumn{4}{c}{0.92} \\
  $r(\text{MC},\text{Task})$     & \multicolumn{4}{c|}{0.87} & \multicolumn{4}{c|}{0.86} & \multicolumn{4}{c}{0.94} \\
  \hline
  \end{tabular}
  }
  \vspace{1mm}
  \caption{\textbf{Cross-domain tokenizer diagnostics and task utility.} For each tokenizer, we report rFID and our two diagnostics (AvgIG, local MC) computed on $(E,D)$, plus downstream task performance on \emph{decoded} images (Remote Sensing: segmentation accuracy/mIoU; Medical: Dice/mIoU; OCR: accuracy). The bottom block reports Pearson correlations across tokenizers, showing AvgIG/MC correlate more strongly with task utility than rFID.}
  \label{fig:tokenizer_task_utility_corr}
  \vspace{-2mm}
  \end{table*}

\subsection{Downstream Training of $G$: MaskGIT with a Diffusion Token Head}
\label{sec:G_train}

After training the tokenizer $(E,D)$, we train a separate generator $G$ on the tokenizer token space.
Given $x\sim\mathcal{D}$, the encoder produces a continuous token sequence
\begin{equation}
z = E(x) = (z_1,\dots,z_K),\qquad z_k\in\mathbb{R}^{d}.
\end{equation}

\paragraph{MaskGIT backbone.}
We adopt MaskGIT to model masked token prediction via a bidirectional transformer $G_\theta$.
For a random mask $m\in\{0,1\}^K$, the backbone takes the visible tokens and mask indicators and outputs per-position context features: $h_k = G_\theta(z_{\neg m}, m, c)_k$.
$c$ is an optional conditioning signal (e.g., class label or text embedding).

\paragraph{Continuous tokens: diffusion regression head.}
Since $z_k$ are continuous, we replace the discrete softmax head with a small diffusion MLP head $M_\psi$ that generates each masked token from noise conditioned on $h_k$.
Using a standard forward noising process $z_k^t = \alpha_t z_k^0 + \sigma_t \epsilon$, $\epsilon\sim\mathcal{N}(0,I)$,
we train $(G_\theta,M_\psi)$ to predict $\epsilon$:
\begin{equation}
\mathcal{L}_{G}(\theta,\psi)
=
\mathbb{E}_{z=E(x),\, m,\, t}
\left[
\sum_{k:\, m_k=1}
\left\|
\epsilon - M_\psi(z_k^t,\, h_k,\, t)
\right\|_2^2
\right].
\label{eq:maskgit_diff_loss}
\end{equation}
As this generator predicts $\epsilon$ for each masked token, we term it as \armodelname.
At inference, \armodelname iteratively refines a partially masked sequence for $S$ steps: it masks uncertain positions, samples their token values using the diffusion head conditioned on $h_k$, and repeats until all tokens are filled, yielding $\hat{z}$.
An image is then obtained by decoding with the tokenizer diffusion decoder: $\hat{x}\sim D(\cdot\mid \hat{z})$.

\section{Experiments}
\label{sec:experiments}

Our experiments validate two claims: (i) \modelname improves tokenizer quality and downstream generation, and (ii) our two metrics, AvgIG and MC, predict downstream generation and task utility more reliably than rFID.
Experiment details are presented in the appendix.

\subsection{Comparisons with state-of-the-art tokenizers}
\label{sec:exp_sota}

We compare \modelname against a suite of state-of-the-art tokenizers spanning both (i) 2D grid tokenization methods and (ii) prior 1D token sequence approaches.
For each tokenizer, we report rFID and our generator-free diagnostics (AvgIG and MC computed on $(E,D)$).
The gFID of the downstream generator trained on the token space of the tokenizer is also reported.
The results are shown in Table~\ref{tab:tok_comp}.
\modelname yields a strong rFID and notably better AvgIG and MC than the state-of-the-art tokenizers.
Additionally, \armodelname trained on \modelname tokens achieves the best gFID among the generator variants. 
This improvements indicates that \modelname combined with \armodelname can be considered as a strong tokenizer and a generator on ImageNet class-conditional generation.
This improvement in gFID is not explained by rFID alone: we observe cases where tokenizers with similar (or even better) rFID underperform in gFID, for example, Semanticist~\cite{semanticist} and FlexTok~\cite{flextok}, consistent with the rFID--gFID mismatch reported in prior works~\cite{vavae,vitok}.
In the same time, we can observe that the AvgIG and MC of \modelname outperforms the tokenizers with similar rFIDs, which means that AvgIG and MC are more predictive of gFID than rFID.

\subsection{Ablations}
\label{sec:exp_ablation}

We ablate each of the three proposed training components in \modelname:
(i) mutual-information token supervision $\mathcal{L}_{\mathrm{MI}}$,
(ii) swap-regularized recoverability $\mathcal{L}^{\mathrm{swap}}_{\mathrm{MI}}$,
and (iii) the Adversarial Flow Model manifold constraint $\mathcal{L}_{\mathrm{AFM}}$.
For each ablation, we report AvgIG, local MC, rFID, and gFID using the same training budget.

\begin{figure*}[t]
  \centering
  \includegraphics[width=.9\linewidth]{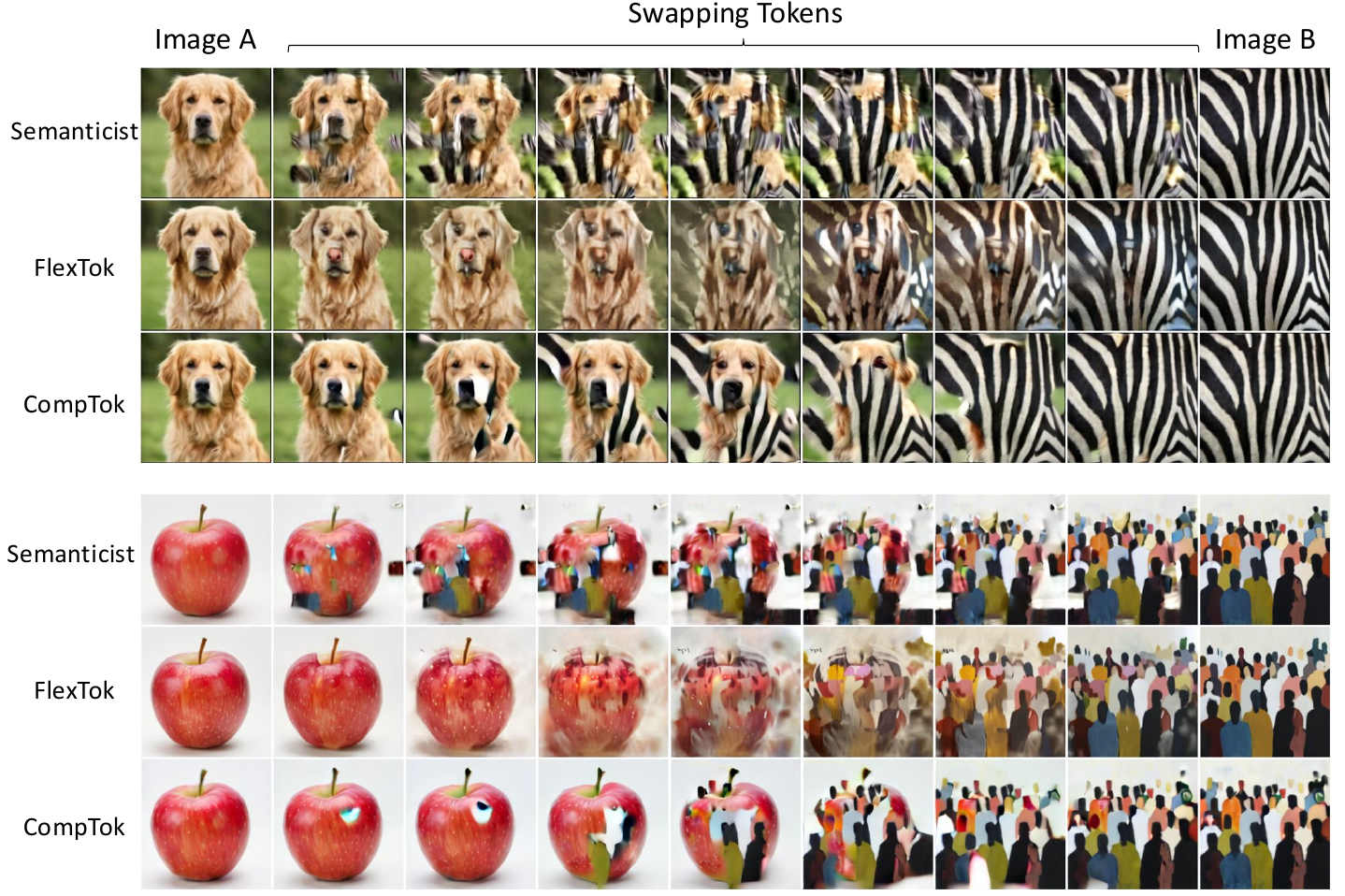}
  \caption{Token swapping as a compositionality probe. For each row, we progressively swap an increasing subset of tokens from Image B into Image A. \modelname preserves coherent structure while transferring semantics, e.g., zebra stripes onto a dog; crowd appearance onto an apple, whereas baselines exhibit mixing artifacts and off-manifold blends.}
  \label{fig:qualitative_results}
\end{figure*}

\begin{wraptable}{r}{0.5\textwidth}
  \vspace{-1mm}
  \centering
  \setlength{\tabcolsep}{5pt}
  \renewcommand{\arraystretch}{1.12}
  \begin{tabular}{l|cccc}
  \hline
  \textbf{Variant} & \textbf{rFID}$\downarrow$ & \textbf{AvgIG}$\uparrow$ & \textbf{MC}$\uparrow$ & \textbf{gFID}$\downarrow$ \\
  \hline
  Full \modelname & 1.27 & 0.43 & 0.74 & 1.60 \\
  w/o $\mathcal{L}_{\mathrm{MI}}$ & 1.89 & 0.21 & 0.71 & 2.45 \\
  w/o $\mathcal{L}^{\mathrm{swap}}_{\mathrm{MI}}$ & 1.42 & 0.29 & 0.62 & 2.68 \\
  w/o $\mathcal{L}_{\mathrm{AFM}}$ & 1.93 & 0.40 & 0.67 & 2.71 \\
  \hline
  \end{tabular}
  \vspace{1mm}
  \caption{\textbf{Ablations of \modelname on ImageNet.} We ablate each proposed loss term and report rFID, our metrics (AvgIG, MC), and downstream gFID after training the same \armodelname generator $G$ on the resulting token space.}
  \label{tab:ablation_imagenet_gfid}
\end{wraptable}

The results are shown in Table~\ref{tab:ablation_imagenet_gfid}.
Removing $\mathcal{L}_{\mathrm{MI}}$ significantly reduces AvgIG: if tokens are not required to be recoverable from decoded samples, some token coordinates become weakly coupled to the decoder output, yielding poor conditioning and lower information gain during latent optimization.
In contrast, rFID does not change that drastically, reflecting that encoder-manifold reconstruction alone does not enforce token usefulness.
$\mathcal{L}_{\mathrm{MI}}^{\mathrm{swap}}$ matters for MC as dropping it would increase the loss barrier between valid tokens.
Removing $\mathcal{L}_{\mathrm{AFM}}$ also leads to a drop in MC.
This aligns with the intuition that AFM provides an adversarial manifold constraint~\cite{lin2025adversarial} that penalizes off-manifold generations, thereby smoothing the decoded landscape around swapped-token regions and reducing barriers.
Again, rFID alone does not reliably predict these effects.

\subsection{Correlations with Downstream Task Utility}
\label{sec:exp_taskcorr}
A useful tokenizer should preserve relevant semantics under encode--decode, not just pixel-level similarity.
Reconstruction metrics such as rFID can be dominated by low-level artifacts and may fail to reflect whether decoded images remain informative for downstream perception tasks.
We therefore test whether AvgIG and MC correlate with task performance more strongly than rFID on three domains not only require pixel-level fidelity but also semantic accuracy.
The three domains are (i) remote sensing, (ii) medical image understanding, and (iii) text OCR.
We use datasets from DeepGlobe~\cite{demir2018deepglobe}, ChestX-ray14~\cite{irvin2019chexpert}, and TextOCR~\cite{singh2021textocr} respectively for the three domains.
For each tokenizer, we run an encode--decode pipeline using $(E,D)$ and evaluate a fixed downstream task model on the decoded images using the task relevant metrics.
We only used a subset of the full test set for each of the tasks to speed up the evaluation process.
Results are presented in Table~\ref{fig:tokenizer_task_utility_corr}.
Across tokenizers, AvgIG and MC correlate more strongly with task performance than rFID.
This supports the interpretation that AvgIG/MC are sensitive to whether token spaces preserve \emph{usable semantic information}: high AvgIG indicates that informative directions in latent token space remain accessible and stable, while high MC indicates that small semantic variations do not map to disconnected or brittle latent regions that would distort fine structures critical for segmentation and OCR.
These results position AvgIG and MC as practical metrics for selecting tokenizers when the end goal is semantic utility, especially in domains where FID-style reconstruction metrics are known to be imperfect proxies.

\subsection{Qualitative Results}
\label{sec:exp_qual}

Figure~\ref{fig:qualitative_results} is a qualitative readout of the same two properties targeted by AvgIG and MC. 
Clean, stable partial swaps imply that many tokens have actionable influence rather than being ignored, which corresponds to higher AvgIG.
Meanwhile, the absence of ``broken" intermediate states indicates that mixed tokens lie in a connected low-barrier region under the realism criterion, aligning with higher MC.
When baselines smear or collapse during partial swaps, \modelname is able to preserve a coherent scene layout while cleanly transferring attributes from Image B onto Image A, yielding realistic intermediate compositions and stable binding.

\section{Conclusion}
\label{sec:conclusion}

We presented \modelname, a diffusion-decoder visual tokenizer designed for downstream token-space generators, and introduced two metrics, AvgIG and MC, that better characterize token-space usefulness and learnability than the reconstruction-only metric of rFID.
Empirically, AvgIG and MC correlate more strongly with downstream gFID and with semantic task utility across remote sensing, medical segmentation, and OCR settings.
Guided by these metrics, \modelname combines mutual-information token supervision, token swapping-based compositional training, and an adversarial flow manifold constraint to reduce token neglect and keep mixed tokens on the real image manifold, yielding improved generation quality by providing a easier to learn latent space for the generator.
Overall, our results suggest that evaluating and designing tokenizers should prioritize the geometry and usability of the induced token space, as captured by AvgIG and MC, rather than reconstruction fidelity alone.

\bibliographystyle{plainnat}
\bibliography{main}

\clearpage

\beginappendix

\section{Implementation Details}
\paragraph{\modelname autoencoder.}
The encoder of \modelname is a standard ViT-B/16~\cite{vit}, except for additional concept tokens and causal attention masks applied to them.
Before being fed to the decoder, the concept tokens are normalized by their own mean and variance following~\cite{RCG}.
The decoder is a DiT~\cite{dit} with a patch size of 2. 
We take DiT-L as default following~\citet{semanticist}. 
The decoder operates on the latent space of a publicly available KL-16 VAE provided by~\cite{mar} to reduce computation cost. 
The VAE is frozen during training, and both the encoder and decoder are trained from scratch.
To enforce the quality of the learned concept tokens and stabilize training, we apply REPA~\cite{repa} with DINOv2-B~\cite{dinov2} as a regularizer to the 8th layer of the DiT decoder.
We do not apply weighting on the loss terms of $\mathcal{L}_{\mathrm{tok}}$, $\mathcal{L}_{\mathrm{MI}}$, $\mathcal{L}_{\mathrm{MI}}^{\mathrm{swap}}$, and $\mathcal{L}_{\mathrm{AFM}}$, as we found this yield the best performance.

\paragraph{Autoregressive image modeling.}
We validate the effectiveness of \modelname autoencoder by training autoregressive image generation models using MaskGIT~\cite{maskgit} combined with diffusion loss~\cite{mar}. 
The input sequence is pre-pended with a \texttt{[CLS]} token for class conditioning, which is randomly dropped out with probability 0.1 during training for classifier-free guidance. 
At inference time, we use a CFG schedule following~\cite{mar,muse}, and do not apply temperature sampling. 
Note that the implementation can be general. 
While our preliminary validation demonstrates promising results, we anticipate better configurations in future work.

\paragraph{AvgIG.} We fix the decoder and, for each evaluation image, initialize $z^0$ from a random latent of the appropriate shape. We run 100 steps of gradient descent on $\tfrac12|D(z)-x|_2^2$ with step size $\eta=0.001$. 
We record $\mathrm{MSE}_t$ and compute $\Delta I_t(x)=\tfrac{N}{2}\log_2\big(\mathrm{MSE}_t/\mathrm{MSE}_{t+1}\big)$ at each step, then average over steps and images. We use a small numerical floor on $\mathrm{MSE}$ to avoid division by zero.

\paragraph{MC.} We set $\mathcal{L}*\psi(x)$ to be the realism loss induced by the adversarial flow model~\cite{lin2025adversarial} trained via $\mathcal{L}*{\mathrm{AFM}}$. On ImageNet, we construct ``nearby" pairs using CLIP image embeddings: for each anchor image, we select (i) the nearest neighbor within the same class and (ii) the nearest neighbor from a different class. 
We report the average MC over all evaluated pairs.

\section{Additional Qualitative Results on Token Swapping}

Figure~\ref{fig:qualitative_results_1},\ref{fig:qualitative_results_2},\ref{fig:qualitative_results_3},\ref{fig:qualitative_results_4},\ref{fig:qualitative_results_5} presents more qualitative results on swapping tokens between images.
From the top row the bottom row of these figures, we present the results of swapping tokens on MaskBit~\cite{maskbit}, a 2D tokenizer, Semanticist~\cite{semanticist}, FlexTok~\cite{flextok}, and \modelname.
We can observe the improved ability of mixing semantics of \modelname comparing to other models.

\begin{figure*}[t]
  \centering
  \includegraphics[width=.9\linewidth]{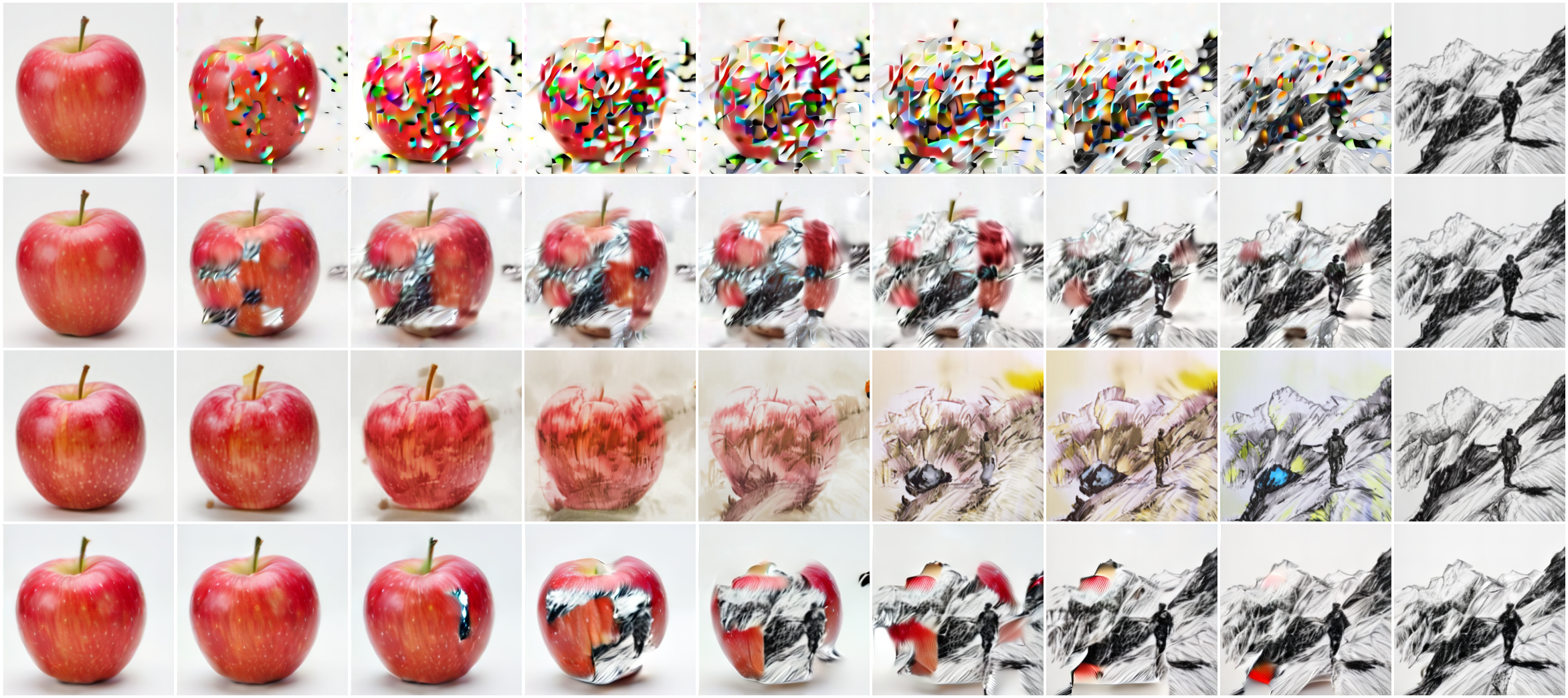}
  \caption{}
  \label{fig:qualitative_results_1}
\end{figure*}

\begin{figure*}[t]
  \centering
  \includegraphics[width=.9\linewidth]{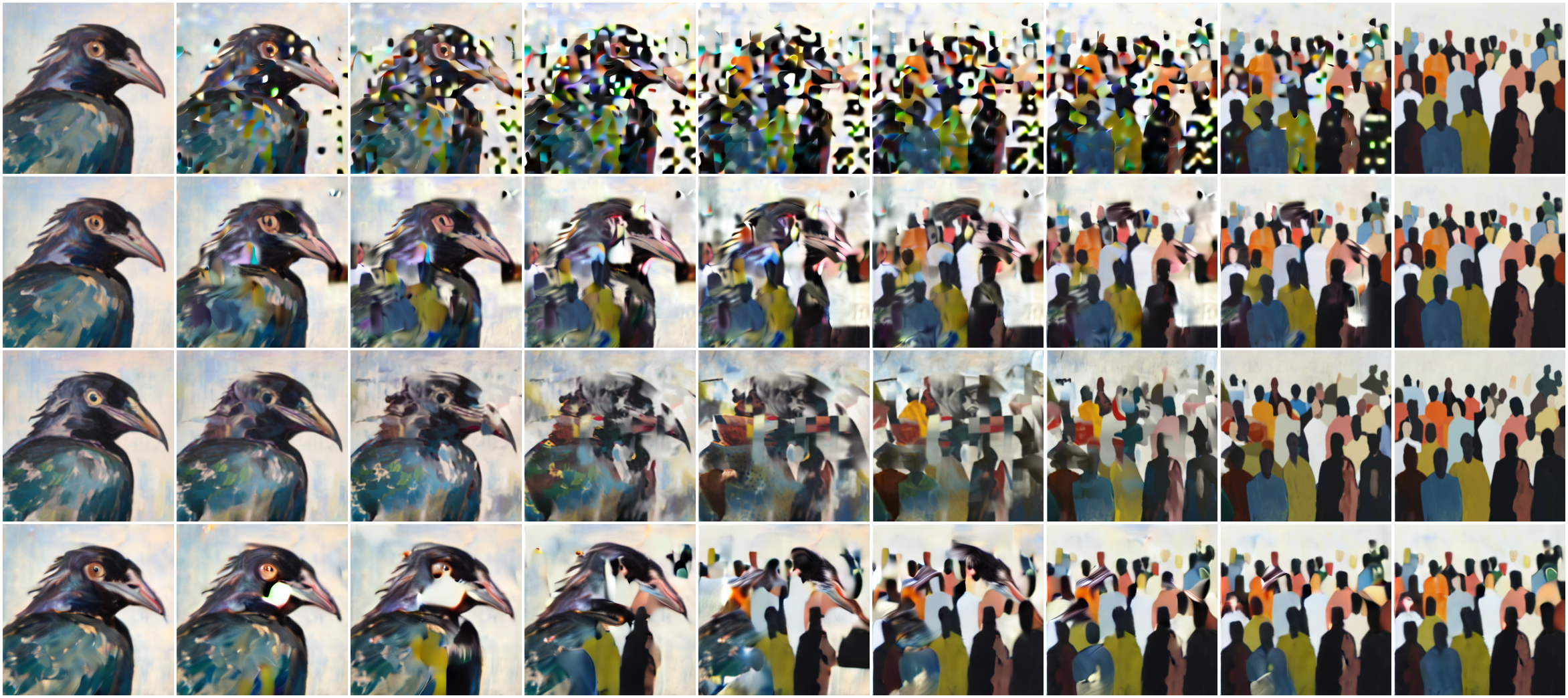}
  \caption{}
  \label{fig:qualitative_results_2}
\end{figure*}

\begin{figure*}[t]
  \centering
  \includegraphics[width=.9\linewidth]{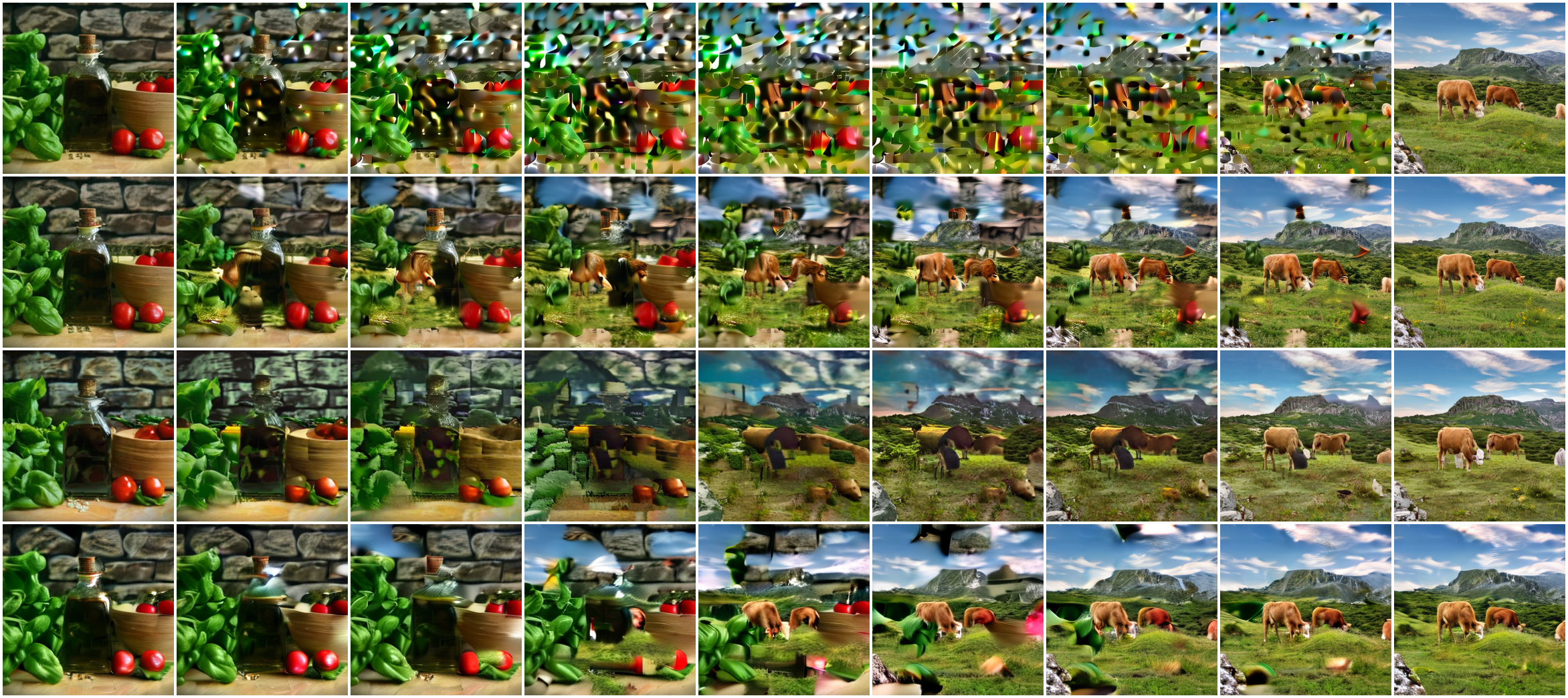}
  \caption{}
  \label{fig:qualitative_results_3}
\end{figure*}

\begin{figure*}[t]
  \centering
  \includegraphics[width=.9\linewidth]{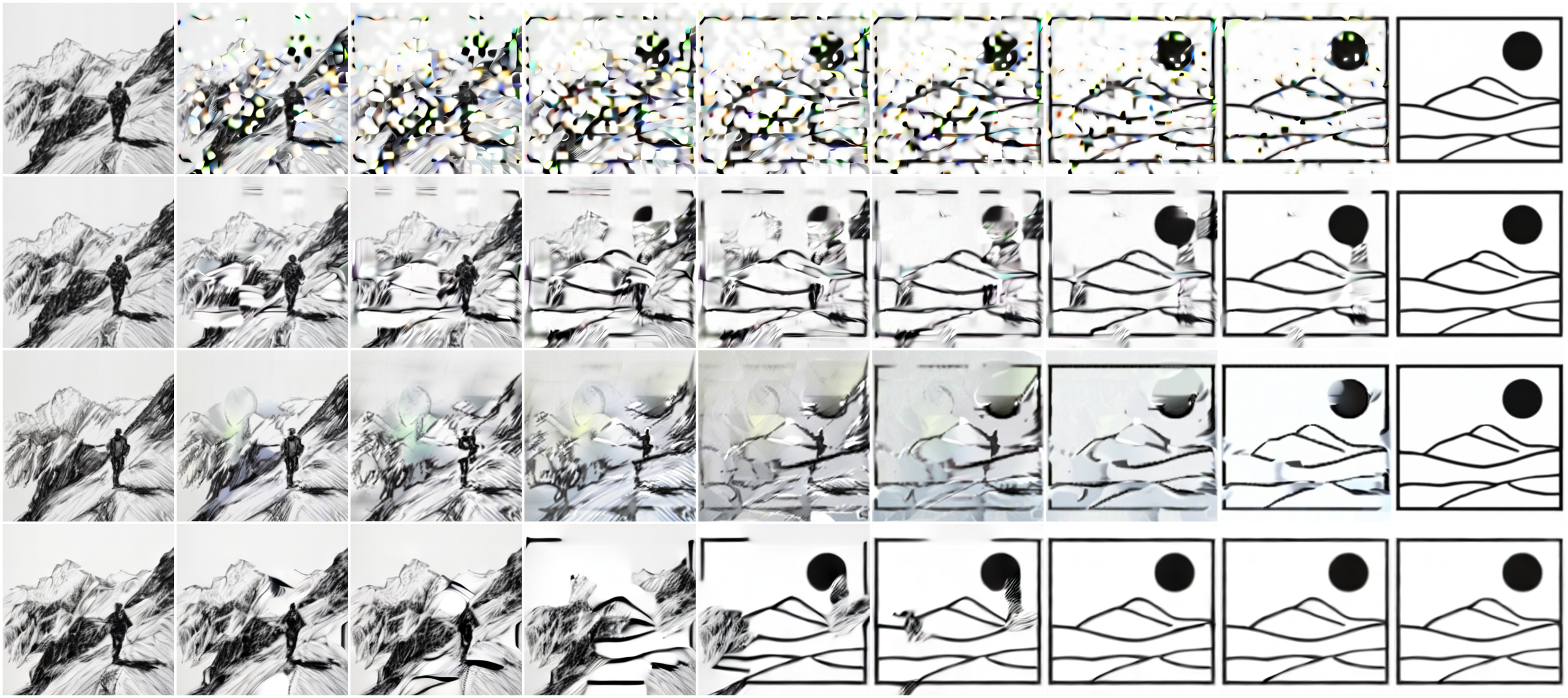}
  \caption{}
  \label{fig:qualitative_results_4}
\end{figure*}

\begin{figure*}[t]
  \centering
  \includegraphics[width=.9\linewidth]{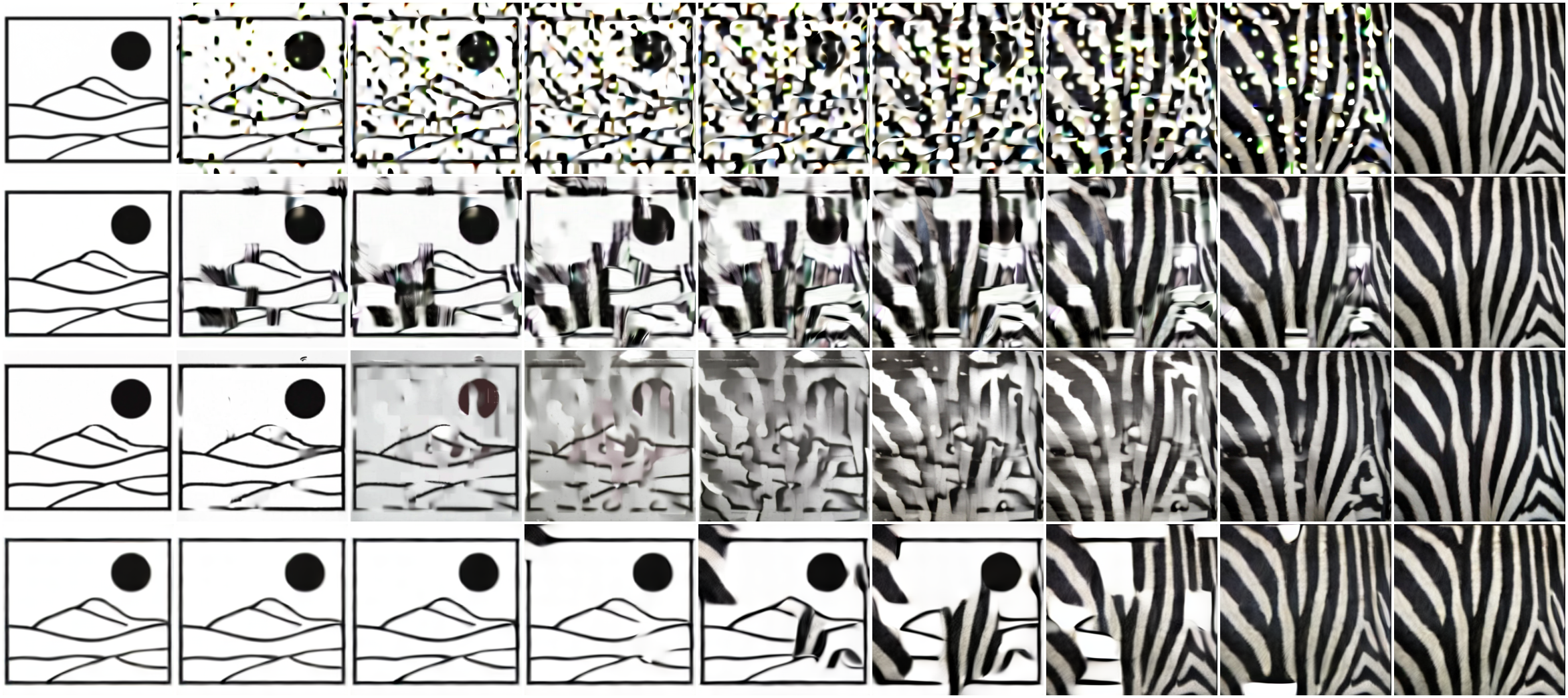}
  \caption{}
  \label{fig:qualitative_results_5}
\end{figure*}

\clearpage

\section{Operational Definitions and Guarantees for AvgIG and Local MC}
\label{sec:appendix_guarantees}

This appendix provides a formal treatment of the two diagnostics introduced in the main text---\textbf{AvgIG} (Average Information Gain) and \textbf{MC} (Mode Connectivity). We provide precise definitions, prove their theoretical interpretations, and formally demonstrate how the proposed training objectives ($\mathcal{L}_{\text{MI}}$, $\mathcal{L}_{\text{swap}}$, and $\mathcal{L}_{\text{AFM}}$) act to optimize these metrics.

\subsection{Setup and Notation}
\label{sec:app_setup}

Let $(E,D)$ be a tokenizer where $E:\mathcal{X}\to\mathcal{Z}$ maps images to a token space $\mathcal{Z} \subset \mathbb{R}^{K \times d}$ (represented as continuous embeddings before quantization or strictly continuous tokens), and $D:\mathcal{Z}\to\mathcal{X}$ is a generator/decoder.
We assume the decoder output $D(z)$ represents the mean of a probabilistic decoding distribution or a deterministic mapping.

\textbf{Realism Loss.} Let $\mathcal{L}_\psi:\mathcal{X}\to\mathbb{R}_{\ge 0}$ be a differentiable realism loss function derived from an auxiliary density model $p_\psi(x)$, such that $\mathcal{L}_\psi(x) = -\log p_\psi(x) + C$. Lower values indicate higher realism.

\textbf{Latent Space Optimization (LSO).} Given a target image $x$, we define the optimization trajectory $\mathcal{T}_x = \{z_t\}_{t=0}^T$ generated by gradient descent on the Mean Squared Error (MSE) reconstruction loss:
\begin{equation}
z_0 = E(x), \quad z_{t+1} = z_t - \eta \nabla_z \mathcal{L}_{\text{rec}}(D(z_t), x), \quad \text{where } \mathcal{L}_{\text{rec}}(x, y) = \frac{1}{2}\|x-y\|_2^2.
\end{equation}

\textbf{Nearby Pairs.} Let $\mathcal{P}_\epsilon = \{(x^A, x^B) \in \mathcal{X}^2 : \|x^A - x^B\|_2 \le \epsilon\}$ be the set of $\epsilon$-nearby image pairs.

\subsection{Operational Definitions}
\label{sec:app_defs}

\begin{definition}[AvgIG: Usable Information Gain]
\label{def:avgig}
For a target $x$ and optimization trajectory $\{z_t\}$, the per-step information gain (in bits) is defined as:
\begin{equation}
\Delta I_t(x) \;=\; \frac{N}{2} \log_2 \left( \frac{\mathcal{L}_{\text{rec}}(D(z_t), x)}{\mathcal{L}_{\text{rec}}(D(z_{t+1}), x)} \right),
\end{equation}
where $N$ is the data dimensionality (number of pixels).
The system-level \textbf{AvgIG} is the expected gain over the data distribution and optimization steps:
\begin{equation}
\mathrm{AvgIG} \;=\; \mathbb{E}_{x \sim p_{\text{data}}}\left[ \frac{1}{T} \sum_{t=0}^{T-1} \Delta I_t(x) \right].
\end{equation}
\end{definition}

\begin{definition}[Local MC: Pairwise Connectivity Ratio]
\label{def:mc}
For a pair $(x^A, x^B) \in \mathcal{P}_\epsilon$, let $z^A, z^B$ be their tokens. Define the linear interpolation path $\gamma(u) = (1-u)z^A + u z^B$ for $u \in [0,1]$.
We define the \emph{Reference Loss} ($L_{\text{ref}}$) and the \emph{Path Maximum Loss} ($L_{\max}$) using the realism function $\mathcal{L}_\psi$:
\begin{align}
L_{\text{ref}}(x^A, x^B) \;&=\; \min \left( \mathcal{L}_\psi(D(z^A)),\, \mathcal{L}_\psi(D(z^B)) \right), \\
L_{\max}(x^A, x^B) \;&=\; \max_{u \in [0,1]} \mathcal{L}_\psi(D(\gamma(u))).
\end{align}
The Pairwise Mode Connectivity score is the ratio:
\begin{equation}
\mathrm{MC}(x^A, x^B) \;=\; \frac{L_{\text{ref}}}{L_{\max} + \delta}, \quad \delta > 0.
\end{equation}
The system-level \textbf{MC} is $\mathbb{E}_{(x^A, x^B) \sim \mathcal{P}_\epsilon} [\mathrm{MC}(x^A, x^B)]$.
\end{definition}

\subsection{Theoretical Guarantees for Diagnostics}

Here we provide formal justification for why these metrics capture ``optimization condition'' and ``manifold smoothness''.

\subsubsection{AvgIG Represents Likelihood Improvement}

\begin{lemma}[AvgIG $\equiv$ Likelihood Gradient]
\label{lemma:avgig_likelihood}
Assume a Gaussian probabilistic decoder $p(x|z) = \mathcal{N}(x; D(z), \sigma^2 I)$. Then, the term $\Delta I_t(x)$ is exactly equal to the increase in log-likelihood of the target $x$ achieved by the update step $z_t \to z_{t+1}$ (shifted by a constant).
\end{lemma}

\begin{proof}
The log-likelihood is $\log p(x|z) = -\frac{N}{2}\log(2\pi\sigma^2) - \frac{1}{2\sigma^2}\|D(z)-x\|_2^2$.
The improvement in log-likelihood (in bits, base 2) is:
\begin{align}
\Delta \text{LL} &= \log_2 p(x|z_{t+1}) - \log_2 p(x|z_t) \\
&= \frac{1}{2\sigma^2 \ln 2} \left( \|D(z_t)-x\|_2^2 - \|D(z_{t+1})-x\|_2^2 \right).
\end{align}
Definition \ref{def:avgig} uses the ratio of MSEs. Note that for small steps, $\log(a/b) = \log(1 + \frac{a-b}{b}) \approx \frac{a-b}{b}$.
AvgIG is a robust, scale-invariant formulation of this likelihood ascent.
\textbf{Implication:} High AvgIG guarantees that the latent space provides \emph{actionable gradients} that rapidly increase the data likelihood. If AvgIG is low, the decoder is locally invariant to the token (vanishing gradients) or the loss landscape is ill-conditioned.
\end{proof}

\subsubsection{Smoothness Implies High MC}

We formally prove that geometric smoothness of the realism landscape implies an MC score close to 1.

\begin{theorem}[Lipschitz Smoothness Lower Bounds MC]
\label{thm:mc_smoothness}
Let $\mathcal{L}_{\text{real}}(z) := \mathcal{L}_\psi(D(z))$ be the composition of the decoder and realism loss.
Assume $\mathcal{L}_{\text{real}}$ is $K$-Lipschitz continuous on the convex hull of the token space:
$|\mathcal{L}_{\text{real}}(z) - \mathcal{L}_{\text{real}}(z')| \le K \|z - z'\|_2$.
Then, for any pair $(x^A, x^B)$ with token distance $\Delta_z = \|z^A - z^B\|_2$:
\begin{equation}
\mathrm{MC}(x^A, x^B) \;\ge\; \frac{1}{1 + \frac{K \Delta_z + |d_{\text{end}}|}{L_{\text{ref}} + \delta}},
\end{equation}
where $d_{\text{end}} = \mathcal{L}_{\text{real}}(z^A) - \mathcal{L}_{\text{real}}(z^B)$ is the endpoint discrepancy.
\end{theorem}

\begin{proof}
Let $\gamma(u) = (1-u)z^A + u z^B$. By Lipschitz continuity:
\begin{align}
L_{\max} &= \max_{u \in [0,1]} \mathcal{L}_{\text{real}}(\gamma(u)) \\
&\le \mathcal{L}_{\text{real}}(z^A) + \max_{u} K \|\gamma(u) - z^A\|_2 \\
&= \mathcal{L}_{\text{real}}(z^A) + K \|z^B - z^A\|_2.
\end{align}
Without loss of generality, assume $\mathcal{L}_{\text{real}}(z^A) \le \mathcal{L}_{\text{real}}(z^B)$, so $L_{\text{ref}} = \mathcal{L}_{\text{real}}(z^A)$.
Then $L_{\max} \le L_{\text{ref}} + K \Delta_z + (\mathcal{L}_{\text{real}}(z^B) - \mathcal{L}_{\text{real}}(z^A))$ if we started from the larger endpoint, but strictly bounded by $L_{\text{ref}} + K \Delta_z$ relative to the start point.
Taking the conservative bound $L_{\max} \le L_{\text{ref}} + K \Delta_z + |d_{\text{end}}|$, we substitute into the MC definition:
\begin{equation}
\mathrm{MC} = \frac{L_{\text{ref}}}{L_{\max} + \delta} \ge \frac{L_{\text{ref}}}{L_{\text{ref}} + K \Delta_z + |d_{\text{end}}| + \delta} = \frac{1}{1 + \frac{K \Delta_z + |d_{\text{end}}|}{L_{\text{ref}} + \delta}}.
\end{equation}
\end{proof}

\textbf{Implication:} As token distance $\Delta_z \to 0$ (which is true for nearby images in a continuous encoder), MC approaches 1. Low MC values prove the violation of the Lipschitz condition---i.e., the existence of sharp, high-loss barriers (``spikes'' or ``holes'') in the latent manifold.

\subsection{Connecting Proposed Losses to Diagnostics}
\label{sec:app_losses_to_metrics_proofs}

We now formally demonstrate how our training objectives optimize these diagnostics.

\subsubsection{Why Mutual Information ($\mathcal{L}_{\text{MI}}$) Improves AvgIG}

\begin{proposition}[MI Maximization Enforces Non-Zero Gradients]
\label{prop:mi_gradients}
Let the decoder be $p_\theta(x|z)$. The gradients used in LSO are $\nabla_z \log p_\theta(x|z)$.
Minimizing the MI loss $\mathcal{L}_{\text{MI}}$ effectively maximizes a lower bound on the gradient norms $\|\nabla_z \log p_\theta(x|z)\|$.
\end{proposition}

\begin{proof}
The InfoGAN-style loss is $\mathcal{L}_{\text{MI}} \approx - \mathbb{E}_{z, x \sim D(z)} [\log q_\phi(z|x)]$.
This optimizes a variational lower bound on the Mutual Information $I(Z; X)$.
Consider the Fisher Information formulation. High mutual information implies that $Z$ strongly predicts $X$.
If the decoder were locally invariant to a token dimension $z_k$ (i.e., $\nabla_{z_k} D(z) \approx 0$), then changing $z_k$ would not change $x$, making $x$ uninformative about $z_k$. This results in $I(Z_k; X) \approx 0$ and high $\mathcal{L}_{\text{MI}}$.
Conversely, minimizing $\mathcal{L}_{\text{MI}}$ forces $I(Z; X)$ to be high, which requires the decoder to be \emph{sensitive} to changes in $z$. Sensitivity is mathematically equivalent to $\mathbb{E}[\|\nabla_z D(z)\|^2] > 0$.
Since AvgIG is driven by the magnitude of the gradient $\nabla_z \mathcal{L}_{\text{rec}}$, ensuring non-zero sensitivity directly increases AvgIG.
\end{proof}

\subsubsection{Why Swap Training ($\mathcal{L}_{\text{swap}}$) Improves Local MC}

\begin{proposition}[Swap Training Bounds Path Loss]
\label{prop:swap_mc}
Let $z^{\text{swap}}$ be a token formed by swapping tokens between $z^A$ and $z^B$. $z^{\text{swap}}$ lies on the boundary of the hyper-rectangle defined by $z^A, z^B$.
Minimizing $\mathcal{L}_{\text{rec}}^{\text{swap}}(z^{\text{swap}})$ minimizes the realism loss at intermediate points along the interpolation path $\gamma(u)$.
\end{proposition}

\begin{proof}
The MC metric penalizes $L_{\max} = \max_{u} \mathcal{L}_\psi(D(\gamma(u)))$.
While $\gamma(u)$ is a continuous diagonal path, the swapped token $z^{\text{swap}}$ represents a ``vertex'' move on the hypercube between $z^A$ and $z^B$.
Let $\mathcal{S}_{AB}$ be the set of all token-swapped tokens between $z^A$ and $z^B$.
Swap training minimizes $\mathbb{E}_{z \in \mathcal{S}_{AB}} [\mathcal{L}_{\text{rec}}(z)]$.
Assuming the decoder is smooth (locally Lipschitz with constant $K$), the loss at any point $\gamma(u)$ on the diagonal is bounded by the loss at the nearest swapped vertices.
Specifically, for any $u$, there exists a $z^{\text{swap}}$ such that $\|\gamma(u) - z^{\text{swap}}\|$ is small.
By minimizing the loss at the dense set of swapped vertices $\mathcal{S}_{AB}$, we lower the upper bound of the loss surface over the entire region, thereby reducing $L_{\max}$ and increasing the MC ratio.
\end{proof}

\subsubsection{Why AFM ($\mathcal{L}_{\text{AFM}}$) Improves Local MC}

\begin{proposition}[Adversarial Alignment Minimizes $L_{\max}$]
\label{prop:afm_mc}
The AFM objective minimizes the Jensen-Shannon divergence (or equivalent adversarial metric) between the distribution of swapped decodes $P_{\text{swap}}$ and real data $P_{\text{data}}$.
Minimizing $\mathcal{L}_{\text{AFM}}$ implies $P_{\text{swap}} \to P_{\text{data}}$ in distribution.
This maximizes the expected MC score.
\end{proposition}

\begin{proof}
Recall the MC definition relies on $L_{\max}$, the maximum realism loss (inverse likelihood) along the path.
High $L_{\max}$ corresponds to ``off-manifold'' artifacts (low probability under $P_{\text{data}}$).
The adversarial generator loss is:
\begin{equation}
\min_D \mathcal{L}_{\text{adv}} = \min_D \mathbb{E}_{z \sim \text{Swaps}} [\text{softplus}(D_\psi(D(z)) - D_\psi(x))].
\end{equation}
The global optimum of this objective is achieved when the distribution of generated swapped images exactly matches the real data distribution ($P_{\text{swap}} = P_{\text{data}}$).
If $P_{\text{swap}} = P_{\text{data}}$, then the support of swapped decodes lies entirely within the high-density region of real images.
Consequently, for any swapped token $z^{\text{swap}}$ (which lies near the interpolation path), the realism loss $\mathcal{L}_\psi(D(z^{\text{swap}}))$ is minimized (approaching the loss of real data).
This minimizes the denominator $L_{\max} + \delta$ in the MC ratio, pushing the MC score towards its theoretical maximum of 1.
\end{proof}

\end{document}